\documentclass[final,11pt]{article}

\usepackage{bbm}

\usepackage[top=1in,bottom=1in,left=1in,right=1in]{geometry}
\usepackage{graphicx}
\usepackage{times}

\usepackage[top=1in,bottom=1in,left=1in,right=1in]{geometry}
\usepackage{graphicx}
\usepackage{parskip}
\usepackage{times}
\usepackage{amsmath, amssymb, amsfonts, cite}
\usepackage{pgf}
\usepackage{tikz}
\usetikzlibrary{arrows,automata}

\usepackage{nicefrac}
\usepackage{xstring}
\usepackage{mathptmx}
\usepackage{tgtermes}
\usepackage{microtype}
\newcommand{\comment}[1]{}
\usepackage{amsthm}
\usepackage[unicode,psdextra,colorlinks=true,citecolor=blue,bookmarksnumbered=true,hyperfootnotes=false,linkcolor=green!40!black]{hyperref}
\usepackage{cleveref}

\usepackage{geometry}
\usepackage{fullpage}

\usepackage{pdfsync, graphicx}

\usepackage{marginnote}
\usepackage{framed}

\usepackage{booktabs,siunitx}

\DeclareMathOperator*{\E}{\mathbbm{E}}

\newtheorem{theorem}{Theorem}[section]
\newtheorem{proposition}{Proposition}[section]

\newtheorem{fact}[theorem]{Fact}
\newtheorem{lemma}[theorem]{Lemma}

\theoremstyle{definition}

\newtheorem{definition}{Definition}[section]

\crefname{theorem}{Theorem}{Theorems}
\crefname{observation}{Observation}{Observations}
\crefname{claim}{Claim}{Claims}
\crefname{condition}{Condition}{Conditions}
\crefname{example}{Example}{Examples}
\crefname{fact}{Fact}{Facts}
\crefname{lemma}{Lemma}{Lemmas}
\crefname{corollary}{Corollary}{Corollaries}
\crefname{definition}{Definition}{Definitions}
\crefname{remark}{Remark}{Remarks}

\renewcommand{\small}{\normalsize}

\renewcommand{\Pr}[2][]{\ensuremath{\mathbb{P}_{#1}\insq{#2}}}
\newcommand{\Ex}[2][]{\ensuremath{\mathbb{E}_{#1}\insq{#2}}}

\newcommand{\inb}[1]{\left\{#1\right\}}
\newcommand{\inp}[1]{\left(#1\right)}
\newcommand{\insq}[1]{\left[#1\right]}

\newcommand{\abs}[1]{\ensuremath{\left|#1\right|}}

\date{}

\usepackage{authblk}

\title{\bf  A Distributed Learning Dynamics in Social Groups}

\author[1]{L. Elisa Celis\thanks{This research was supported in part by an SNF Project Grant (205121\_163385).}}
\author[2]{Peter M. Krafft}
\author[3]{Nisheeth K. Vishnoi}
\affil[1,3]{\small \'{E}cole Polytechnique F\'{e}d\'{e}rale de Lausanne (EPFL), Switzerland}
\affil[2]{\small MIT}

\begin{document}
\maketitle

\begin{abstract}
We study a distributed learning process observed in human groups and  other social animals. 
This learning process appears in settings in which each individual in a group is trying to decide over time, in a distributed manner, which option to select among a shared set of options. 
Specifically, we consider a  stochastic dynamics in a group in which every individual selects an option in the following two-step process: (1) select a random individual and observe the option that individual chose in the previous time step, and (2) adopt that option if its stochastic quality was good at that time step. 
Various instantiations of such distributed learning appear in nature, and have also been studied in the social science literature.
From the perspective of an individual, an attractive feature of this learning process is that it is a simple heuristic that requires extremely limited computational capacities. 
But what does it mean for the group -- {\em could such a simple, distributed and essentially memoryless process lead the group as a whole to perform optimally?} 
We show that the answer to this question is {\em yes} --  this distributed learning is highly effective at identifying the best option and is close to optimal for the group overall. 
Our analysis also gives  quantitative bounds that show fast convergence of these stochastic dynamics. 
We prove our result by first defining a (stochastic) infinite population version of these distributed learning dynamics and then combining its strong convergence properties along with its relation to the finite population dynamics. 
Prior to our work the only theoretical work related to  such learning dynamics has been  either in  deterministic special cases or in the asymptotic setting.
Finally, we observe that our infinite population dynamics is a stochastic variant of the classic multiplicative weights update (MWU) method. Consequently, we arrive at the following interesting converse:  the learning dynamics on a finite population considered here can be viewed as a novel distributed and low-memory implementation of the classic MWU method.
\end{abstract}

\newpage

\section{Introduction}
A powerful assumption -- often leveraged in biology, ecology, and
evolutionary psychology in order to reason about why humans, animals,
and other biological organisms behave in certain ways -- is to suppose
that behavior is tuned to alleviate evolutionary pressure. 
 This assumption provides deductive power because once a behavior is assumed
to be optimally or near-optimally solving \emph{some} problem, one can then attempt to discover \emph{which} problem the system might be solving through the {\em computational lens}.
We study this phenomenon in the context of 
social behavior.
In particular, we consider a class of distributed {social learning} dynamics  which are at once
conspicuous in daily life, oft discussed in the social science
literature, and also  empirically verified; yet are also simple to the point that they appear perhaps suboptimal.  
Consider the setting consisting of a social group of $N$ individuals presented with a set of $m$ options of different quality. 
The quality of each option is assumed to be an independent random variable whose parameters are unknown to the individuals and remain fixed over time.
 At each  time step each individual selects one option and observes a stochastic indicator of that option's quality. 
The goal of the {\em individual} is to identify the best option.

The distributed learning dynamics then boils down to  {\em individuals copying  or imitating the behavior of others} in an effort to solve this problem. 
Such dynamics roughly have the following two steps: At each time step, each individual independently decides which option to select by first 
\begin{enumerate}
\item {\bf Sampling --}  observing the \ choice of a random member of the group at the last time step, and then  
\item {\bf Adopting --} deciding whether or not to adopt the recommended option as a function of the {\em most recent} (stochastic) signal of that option's quality. 
\end{enumerate}
Instances of such two-stage distributed learning dynamics in social settings have been widely proposed and validated with data in the literature on human choice behavior (e.g., \cite{bianconi_bose_2001, krumme_quantifying_2012, mcelreath2008beyond, beheim_strategic_2014, granovskiy2015integration}) and animal behavior (e.g., \cite{pratt_agent-based_2005,seeley_group_1999}).
They are cognitively simple because individuals need not maintain any history of previous observations; rather they only use the most recent quality signal of one option. 
Furthermore, as with many other distributed protocols observed in nature (e.g.,  \cite{GhaffariMRL15, MuscoSL16}), each step requires only limited communication with other group members. 
In light of the fact that such  distributed learning processes are ubiquitous, 
the question arises -- why? 
To answer this, we need to understand the following:  
 {\em What does such a distributed learning dynamics imply for the  group as a whole?}

\subsection*{Our Contribution.}
We consider a general model that captures a wide variety of distributed  learning dynamics in social settings as defined above (see Section~\ref{sec:model} for a formal definition) and study two fundamental algorithmic questions:
\begin{itemize}
\item \emph{do such  learning dynamics (despite each individual having a limited memory) have the potential to successfully converge to the best option for the collective population?} and if so, 
\item  \emph{how efficient are such dynamics?}
\end{itemize}
We prove that the answer to the first  question is {\em yes} and provide quantitative bounds for the second in the most general 
case  in which the population is finite and both the sampling and adopting step are stochastic.

In general, due to the fact that populations are finite and there is stochasticity in both steps, the social learning dynamics may be chaotic, with no single option dominating, and the popularity of options rising and falling over time.
However, we prove that social learning  leads the population, as a whole, to be competitive -- pretty quickly -- as compared to the best strategy in {\em hindsight}; i.e., the dynamics have low {\em regret} (see \cref{sec:model}).
Prior to our work, many instances of  such social learning dynamics have been proposed and validated  (see, e.g., \cite{pentland2014social} for an exposition), and, even though on the surface there seem to be several related processes, the only theoretical work has been either in deterministic special cases or in the asymptotic setting where both the size of the population and time goes to infinity; such results (see  \cite{ellison1995word, bjornerstedt1994nash} and Section \ref{sec:relatedwork}) effectively focus on deriving asymptotic (large deviation) bounds. 
 To the best of our knowledge, ours is the first rigorous analysis of this distributed learning dynamics in a realistic setting when the size of the group is finite.

Key to our results are the following two  realizations 
which we can use to understand the emergent  behavior of the distributed learning dynamics:
\begin{itemize}
\item  In the infinite population limit, by rewriting the underlying stochastic equations of the distributed learning dynamics, the individuals are effectively implementing a {\em stochastic}  variant of the classic multiplicative weights update (MWU) method \cite{arora2012multiplicative} at the group-level.\footnote{This stochastic process is not to be mistaken with the standard {deterministic} MWU or its continuous time limit, the replicator dynamics.}

\item While the finite population distributed learning dynamics can be approximated with its infinite population limit for {\em short} times, to ensure that the regret bounds remain valid for longer times, further new ideas are required. Here, we show how we can appeal to the strong {convergence} properties of the infinite population stochastic dynamics to ascertain that the regret  remains bounded for all times in the finite population case. 
\end{itemize}

Computationally, contrasting with typical implementations of the MWU method, in the  learning dynamics we consider, no  individual keeps track of the weights. 
Rather, the popularity of the options in the previous time step serve as a proxy for weights, and suffice to propel the process forward.
Thus, we arrive at the following interesting conclusion -- {\em the learning dynamics in social groups considered here can inform novel, low-memory, low-communication, distributed implementations of the MWU algorithm in the stochastic setting}; perhaps appropriate for  low-power devices in distributed settings 
such as sensor networks or the internet-of-things. 

\section{Model,  Our Results and Overview}

\subsection{The Model}
\label{sec:model}

The learning environment we consider consists of a set of $N$ individuals repeatedly choosing among $m$ options during a sequence of $T$ discrete time
steps.  
Each option has an unknown underlying quality, $\eta_1 \geq \eta_2 \geq \cdots \geq \eta_m \in [0,1]$, which 
represents the probability that the option is ``good'' at any given time step; we let %
$R_j^t = \mathrm{Bernoulli}(\eta_j)$ be the indicator random variable for the event that option $j$ is good at time $t$. 
The goal of each individual is to select the best option. 
Now, let $X_{ij}^t \in \{0,1\}$ be the indicator random variable for the event that individual $i$ chooses to adopt option $j$ at time $t$. 
The distributed learning dynamics we consider is a two-stage model: 

\paragraph{\bf (1) Sampling.} First, individuals select which option to consider, and then they choose whether or not to adopt that option.  
To obtain an option to consider at time $t+1$, with probability $\mu$ individual $i$ selects an option $j \in [m]$ to consider uniformly at random, and with probability $(1-\mu)$ individual $i$ selects an option $j \in [m]$ proportional to its current popularity: 
$$Q_{j}^t = \frac{\sum_{i=1}^N X_{ij}^t}{\sum_{k=1}^m \sum_{i=1}^N X_{ik}^t}$$
is the fraction of the population that adopts option $j$ at time $t$, and we assume $Q_j^0 = \frac 1 m$ for all $j$.\footnote{In the absence of prior knowledge, we initialize at the point where all options are equally popular. This also simplifies the exposition, but is not crucial to our results -- our results hold from arbitrary initial conditions.} 
Note that this can be implemented in a distributed manner as suggested in the introduction by letting $i$ select a companion $i' \in \{1,\ldots,N\}$ uniformly at random, and observing the choice of individual $i'$ at time $t$.
The parameter $\mu > 0$ is small and represents a fraction of the population which may take independent decisions; its role  is to ensure that the population does not get stuck in a bad option.

\paragraph{\bf (2) Adopting.} In the second stage, after choosing an option $j$ to consider, individual $i$ must then decide whether to commit to this option or to sit out during this time step.
Individual $i$  observes the most recent quality signal associated with $j$, namely $R_{j}^{t+1}$  and decides to adopt the option with probability $f_i(R_j^{t+1}) \in \{0,1\}$ where $f_i$ is a stochastic function such that $\E[f_i(1)] > \E[f_i(0)]$. 
Hence, for each $i$, we can express $f_i$ in the following form:  
\[
f_i(R_j^{t+1}) = \left\{
\begin{tabular}{ll}
1 & {with probability $\beta_i$  if $R_j^{t+1} = 1$} \\
1 & {with probability $\alpha_i$ if  $R_j^{t+1} = 0$} \\
0 & otherwise.
\end{tabular}
\right.
\]
Here, $\alpha_i \leq \beta_i$ are parameters of the model and represent how sensitive individuals are to the most recent signal of goodness as compared to the weight they give to the recommendation.
For simplicity in the exposition, we assume that all $f_i$ are identical, and drop the index $i$. This assumption is not essential for our results 
-- we omit the details.
Thus, the two relevant parameters are $0 \leq \alpha \leq \beta \leq 1$.

\paragraph{\bf Examples of the model.}

Many instances of social learning in the social sciences and economics literature can be interpreted as special cases of the distributed learning dynamics introduced above. Here we give two concrete examples -- one direct and one indirect; see more discussion in Section \ref{sec:relatedwork}.
The simplest such  example \cite{krafft2016human}  corresponds exactly to our model when $\alpha = 1-\beta$ for some $\beta \geq \frac 1 2$ when $\eta_1 > \frac 1 2 = \eta_2 = \cdots = \eta_m$. The authors validate this model using observational data on the decisions of amateur investors on an online platform in which users are able to copy the actions of others. 

Another instance, which takes a bit of explanation, appears in  the economics literature \cite{ellison1995word}. 
We present it here because it illustrates two common ways in which more general-looking models, specifically ones with continuous-valued rewards and reward differences across individuals, can often be reinterpreted in such a way that our framework applies. 
The authors consider a learning setting where $m = 2$ and rewards $r_j^t $ are drawn from a continuous-valued distribution $\mathcal F_j$.
Furthermore, their model incorporates player-specific stochastic shocks, so that if  $\epsilon_{ij}^t \sim \mathcal G$ is the size of the shock to player $i$ on option $j$ at time $t$, then the reward to player $i$ is $r_j^t + \epsilon_{ij}^t$.
 The sampling step (1) is similar, except that $\mu = 0$. 
In the adoption step (2), if player $i$ sampled player $i'$, then player $i$ adopts option $1$ if $r_1^t + \epsilon_{i1}^t + \epsilon_{i'1}^t > r_2^t + \epsilon_{i2}^t + \epsilon_{i'2}^t$, and adopts option $2$ otherwise.  
To convert this to our setting, let $R_1^t$  (and $R_2^t$) be the indicator random variable for the event that $r_1^t > r_2^t$ (and $r_1^t < r_2^t$); this occurs with some probability $p$ (and $1-p$) and defines our parameters $\eta_1 = p > 1-p = \eta_2$.\footnote{Note that in their model the $R_1^t$ and $R_2^t$ are correlated as exactly one of them is 1 in every time step; however independence across $t$ remains which suffices for our results.}
Note that in the adoption step, we can replace $\epsilon_{i1}^t + \epsilon_{i'1}^t - \epsilon_{i2}^t - \epsilon_{i'2}^t$ by a continuous random variable $\xi$.
Then $\beta =\Pr{\xi > r_2^t - r_1^t | r_1^t > r_2^t}$ and $\alpha =\Pr{\xi > r_2^t - r_1^t | r_2^t > r_1^t}$. 
As the $\epsilon_{ij}^t$s are i.i.d., $\xi$ has zero mean and is symmetric, hence $\alpha < \beta$ and our results apply.
The authors consider the  infinite population version of this model where a constant fraction of the population updates at each time step,  and analyze its asymptotic properties as $T \to \infty$.


\subsection{Our Results and Overview}

In order to explain our results for the distributed learning dynamics we first need to quantify a measure of optimality. 
In the remainder of the paper we  let $\alpha=1-\beta$. This just simplifies the reading and has no impact on the statements of the theorem -- the same bounds hold with a dependence on $\frac{\beta}{\alpha}$ as opposed to $\frac{\beta}{1-\beta}.$
Let $Q_j^t$ be the random variable which measures the fraction of individuals that chose option $j$ at time $t$, as defined above. 
Wishfully, we might like to prove that as $t$ becomes large, $Q_j^t$s, for all $j\neq1$ are close to zero (assuming $\eta_1>\eta_2$).
However, simple examples show that this may not be the case; the stochastic process is non-monotone, even when there is a significant gap between $\eta_1$ and $\eta_2$, and may step away significantly from $Q_1^t \approx 1$ even for large $t$. 
Instead, we consider the average expected performance of the group when compared to that of the best option: 
 \[ \mbox{Regret}_{N}(T) :=  \eta_1 - \frac 1 T \sum_{t=1}^T \sum_{j=1}^m \E[Q_j^{t-1}R_j^t]. \]
As the name suggests, this is nothing but the average regret of this process; namely, the difference between the group's expected average reward if all individuals who adopted an option select the optimal $j = 1$, and the group's expected average reward selected according to the distributed learning process up to time $T$. 
The following is the main technical contribution of the paper;  see \cref{thm:main2} for a formal statement of this result.

\emph{\textbf{Distributed learning achieves near-optimal regret in a social group.}
For a range of parameters $\nicefrac{1}{2} \leq \beta \leq \nicefrac{e}{(e+1)}$,  $\mu$ a small constant, $N$ roughly at least  $m^{\nicefrac{1}{\delta^2}}$ and for all $T \geq \frac{\ln m}{\delta^2}$,  $\mbox{Regret}_{N}(T)$ is at most   $6\delta$ where $\delta = \ln\inp{\frac{\beta}{1-\beta}}$.}
 Thus, the closer $\beta$ is to $\nicefrac 1 2$, the better the regret.

The proof of the above result relies on the following  connection between our distributed learning dynamics in a finite population and what can be thought of as an infinite population variant of the distributed learning dynamics.
The latter can also be seen as a stochastic variant of the MWU method (see \cref{lem:main}) which we explain below.
Consider  $m$ experts where expert $j$ generates a stochastic reward $R_j^{t+1}$ at time $t$ that is $1$ with probability $\eta_j$ and $0$ otherwise. In this setting, a {\em single} player maintains weight $W_j^t$ for option $j$ at time $t$, which is updated multiplicatively in the following manner:\footnote{It is worth noticing the similarity between the weights  update (in particular the first term) and that used in the  result of Christiano {\em et al.} \cite{ChristianoKMST11} who developed a variant of the MWU method in the design of a fast algorithm for a flow problem.}

$$W_j^{t+1} := \underbrace{\left( (1-\mu) W_j^t + \frac{\mu}{m}  \sum_{k=1}^m W_k^t \right)}_{\mathrm{deterministic \; sampling}} \; \underbrace{\beta^{R_j^{t+1}}(1-\beta)^{1-R_j^{t+1}}}_{\mathrm{stochastic \; rewards}}$$

and $W_j^0 = 1$ for all $j$. 
We arrive at this update equation by the (non-rigorous) thought process that in the infinite population case,  we can replace the stochastic quantities by their expectation in the sampling stage: 
in this case, the expected fraction of individuals picking the option $j$ in the sampling stage of the social learning dynamics is proportional to $(1-\mu)W_j^t+ \frac{\mu}{m}  \sum_{k=1}^m W_k^t $.
This gives us the first term in the right-hand side above. The second term is just $f(R_j^{t+1})$ with respect to parameter $\beta.$ 
We note that, since the rewards are stochastic, it is not the standard adversarial MWU setting and, since all the information is known to the population as a whole, it is also not the standard setting of stochastic bandits.

Even though these weight update equations are arrived at by a heuristic calculation, we can prove that for {\em short} time periods, there is a coupling between the infinite and the finite distributed learning dynamics such that the stochastic trajectories corresponding to the weights in the infinite population case remain close to that of the finite population case. 

\emph{\textbf{Infinite vs. finite population distributed social learning.} 
If $(P_j^t)_{j=1}^m$ is the probability distribution induced by the weights $W_j^t$ after time $t$ for a given sequence of rewards, then for all $j$, $1- \frac{5^t}{\sqrt{N}} \leq  \nicefrac{P_j^t}{Q_j^t} \leq 1+ \frac{5^t}{\sqrt{N}}$.}

The proof of this crucially relies on the fact that $\mu$ is strictly positive.
On the other hand, the fact that $\mu>0$ also makes the analysis quite messy.
Note that the closeness deteriorates very quickly and, in particular, the bound becomes uninteresting after about $\log N$ time steps.
On the other hand, for fixed $t$, as $N \rightarrow \infty$,  the trajectories of the two processes are identical; put another way, the distributed learning process over infinite populations is identical to the corresponding stochastic MWU method.
This is typically the kind of asymptotic result that exists in the literature.
The more interesting and challenging direction is to show convergence when $N$ is large, but {\em fixed}, and $T$ goes to infinity.

In order to leverage this connection between infinite and finite population variants of the distributed learning dynamics, we first analyze the corresponding regret of the stochastic MWU method, which is defined to be
\[ \mbox{Regret}_{\infty}(T):=  \eta_1 - \frac 1 T \sum_{t=1}^T \sum_{j=1}^m \E[P_j^{t-1}R_j^t]. \] 
We can establish the following regret bound in this case;  see \cref{thm:main1}. 

\emph{\textbf{Infinite population distributed social learning dynamics achieves near-optimal regret.}
For a range of parameters $\nicefrac{1}{2} \leq \beta \leq \nicefrac{e}{(e+1)}$,  $\mu$ a small constant,  and for all $T \geq \frac{\ln m}{\delta^2}$,  $\mbox{Regret}_{\infty}(T)$ is at most $3\delta$ where $\delta = \ln\inp{\frac{\beta}{1-\beta}}$.}

The proof of this obtained by adapting the  proof of the MWU method to take into account stochastic rewards. 

At this point we would be done except that the closeness between the probability distributions of finite and infinite distributed learning dynamics holds only for short times -- if we run the process for about $\frac{\ln m}{\delta^2}$ steps, then for the probability distributions to be close we would need $N \geq m^{O(\nicefrac{1}{\delta^2})}.$
What about when $T \geq \frac{\ln m}{\delta^2}$?
To tackle this problem, we need a new idea.
The first observation is that, in fact, the regret bound for the infinite population case can be made stronger: roughly, {\em as long as the starting distribution has enough entropy, the regret becomes small in about the same number of steps.}
Thus,  we can break $T$ into epochs of size approximately $\frac{\ln m}{\delta^2}$
and observe that at the beginning of each epoch,  each option will have about $\nicefrac{\mu}{m}$ probability -- again we need crucially that $\mu>0$.
Thus, we can show convergence starting from such a probability distribution in $\frac{\ln (\nicefrac{m}{\mu})}{\delta^2}$ iterations.
As a consequence, in each epoch, we can bound the regret by about $6 \delta $ for slightly larger $N$.
Rewriting it a different way, having a finite population will have an additive error term of approximately $\frac{m^{\nicefrac 1 {\delta^2}}}{\sqrt N}$ to the regret obtained by the infinite population. 
Thus, we rely on the strong {\em attractive} properties of the infinite population stochastic dynamics to obtain quantitative regret bounds for the finite population social learning dynamics.
This contributes to the growing set of connections between  using attractors of dynamical systems to analyze stochastic processes\cite{pemantle1991touchpoints,wormald1995differential,BenaimHS05,VishnoiSOE}.
Details of the proof of \cref{thm:main2} appear in Section \ref{sec:main}. 

\section{Related Work}
\label{sec:relatedwork}
There is a growing body of work in theoretical computer science, and distributed computing theory in particular, that studies problems arising in the sciences through the computational lens; see, e.g., \cite{Chazelle,Afek183,MuscoSL16,CornejoDLN14,GhaffariMRL15,straszak2016IRLS}.
Such studies have also on occasion contributed back to computer science by providing insights into existing techniques, or giving rise to novel bio-inspired algorithms. 
Our work touches upon both of these aspects. 
Among related studies in the sciences, while dynamics that \emph{only} have a sampling step \cite{boyd1988culture, estlund1994opinion} or \emph{only} an adoption step \cite{henrich2001cultural} have been studied, combining both has been shown empirically to result in a better strategy \cite{mcelreath2008beyond}. 
Indeed, in line with  these observations, one can formally see from our analysis that if we only have sampling ($\beta = 1-\alpha = 1$) or only have adoption ($\mu = 1$), the process does not always converge to the best option.
Hence, both steps of the process seem crucial, and many models in sociology and economics are such distributed two-step processes \cite{krafft2016human, rogers1988does, banerjee1992simple, ellison1995word, cabrales2000stochastic}.
While some models a priori look different, many can be captured by our formulation; for example, models that have continuous rewards but whose adoption rule depends on whether the reward is above or below a threshold  \cite{bjornerstedt1994nash, gale1995learning, binmore1997muddling} can be converted to a binary reward structure in a standard way. Similarly, differences across individuals can be captured in the functions $f_i$ (see the second example in Section~\ref{sec:model}). 
While some of these models consider the aggregate popularity of options over time, many (including models for human behavior, e.g., \cite{krafft2016human}) consider only the current popularity.
In the economics literature, some similar finite population models have been studied, many of which also fall into our framework. 
Their analysis has only been \emph{asymptotic} as $N, T \to \infty$; such results (see, e.g., \cite{ellison1995word, bjornerstedt1994nash, benaim2003deterministic}) effectively focus on deriving large deviation bounds. 
In contrast, the main technical contribution of our work includes quantitative bounds of the social learning dynamics for finite populations ($N < \infty$).  
Asymptotic and infinite-population results follow as corollaries.

There has been a large body of work on the distributed consensus problem; see for instance \cite{AngluinADFP06,Becchetti:2016} and the references therein. 
The goal in such problems is for all individuals to agree on a single opinion, and various distributed dynamics for doing so have been proposed and analyzed.  
Our setting differs in that there is additional information -- the repeated stochastic signals associated to the quality of each opinion.

In evolutionary game theory, similar-looking deterministic {imitator} dynamics have been considered (see \cite{sigmund2011evolutionary,nowakbook} for an overview). 
A key difference with this work is the learning environment -- we are not attempting to select a strategy in a game, rather are trying to identify the best option of out a collection. 
Hence, the reward of an individual depends on her choice $j$, while in evolutionary game theory the reward of an individual depends on the choices of \emph{the entire population}. In this setting one can still consider the regret of an individual (see, e.g., \cite{blum2008regret}). 
While fast convergence of similar dynamics has been shown for some special cases (e.g., potential games \cite{duersch2011once} and selfish routing \cite{fischer2004evolution}), for general games we cannot expect to always converge quickly unless $PPAD \subseteq P$.  
In fact, it has been shown that in some games, versions of the replicator dynamics may converge to outcomes that are not optimal, or not even equilibria \cite{gale1995learning, samuelson1992evolutionary}. 

The fact that an MWU-like method emerges from a simple distributed behavior of individuals in a social setting, is somewhat reminiscent of unrelated results  that arise  in the context of biological evolution \cite{chastain2014algorithms} and task allocation among ants \cite{su2015algorithms}.
The MWU algorithm \cite{arora2012multiplicative}, or its well-known continuous time limit, the replicator dynamics \cite{taylor1979evolutionarily, smith1982evolution} can also be seen as a special case of our distributed learning dynamics if we remove the randomness from {\em both} the sampling and adopting steps \emph{and} the rewards (effectively taking the process to a deterministic and infinite setting).  
However, as-is, MWU-type dynamics are different because each individual must effectively maintain full weights (or a mixed strategy vector) at every time step. 
Furthermore, the lack of stochasticity and the infinite population setting avoid the key technical hurdles in the analysis of our model. 

Our results suggest that the distributed learning dynamics in finite populations can be viewed as a novel distributed and approximate implementation of the MWU method.
While parallelized implementations for solving multi-armed bandit problem exist (see, e.g., \cite{golovin2010online,nayyar2015regret,liu2010decentralized,gai2011decentralized,anandkumar2011distributed}), in such works each node explicitly maintains a weight vector on all options. 
The most distinctive aspect of the distributed MWU interpretation of the learning dynamics we consider is that no such memory is required -- the weights are represented implicitly by the popularity of the various options, and the sampling and adopting processes require almost no memory. 
This difference distinguishes our distributed learning dynamics from prior work on distributed MWU or bandit methods.  

\section{Technical Details and Proofs}
\subsection{Basic Facts}

We first recall a few theorems and definitions that will be useful in our proofs.

\begin{theorem}[\textbf{Chernoff-Hoeffding bounds~\cite{dubhashi2009concentration}}]
  \label{thm:chernoff}
  Let $Z_1, Z_2, \dots, Z_n$ be independent Bernoulli random variables with
  mean $\gamma_i$.  Let $\gamma:=\frac{1}{n} \sum_{i=1}^n \gamma_i.$ When $0 < \delta \leq 1$, we have
$      \Pr{\abs{\frac{1}{n}\sum_{i=1}^n Z_i - \gamma} > \gamma\delta}
      \leq 2\exp\inp{-\nicefrac{n\gamma\delta^2}{3}}.
  $
 \end{theorem}

\begin{definition}
For real numbers $A$, $B$, and $c \geq 0$, the notation  
$A \stackrel{c}{\sim} B$ denotes
$  \frac{1}{c} \leq \frac{A}{B} \leq c.$ 
\end{definition} 

\begin{fact}\label{fct:exp}
For fixed $0 \leq \delta \leq 1$ and for all $0 \leq x \leq 1$, 
$ e^{\delta x} \leq 1+ (e^\delta -1) x.$
\end{fact}

\subsection{The Infinite Population Distributed Learning Dynamics}
Consider the following stochastic process: 
$W_j^0:=1$ for all $1\leq j \leq m$.
For $t > 0$ and for all $1 \leq j \leq m$
\begin{equation}
W_j^{t+1} :=  \left( (1-\mu) W_j^t + \frac{\mu}{m}  \sum_{k=1}^m W_k^t \right) \; \beta^{R_j^{t+1}} (1-\beta)^{1-R_j^{t+1}}.
\end{equation}
This definition parallels the two-step finite population distributed learning dynamics and can be thought of as an infinite population distributed learning dynamics: $W_j^{t+1}$ is first updated according to $W_j^t$ (with weight $1-\mu$) and with uniform additive factor $\nicefrac{\sum_{k=1}^m W_k^t}{m}$ (with weight $\mu$), and then (stochastically) as the corresponding function of $\beta$ and $R_j^{t+1}$.
The stochasticity is now solely with respect to the $R_j^t$s.
Now, consider the following probability distribution corresponding to these weights,
$P_j^t := \frac{W_j^t}{\sum_{k=1}^m W_k^t},$ which corresponds to the fraction of the (infinite) population that has adopted option $j$ at time $t$.

Let $\delta := \ln  \inp{\frac{\beta}{1-\beta}}$, and note that we can re-write the above as
$$W_j^{t+1} :=  (1-\beta)\left( (1-\mu) W_j^t + \frac{\mu}{m}  \sum_{k=1}^m W_k^t \right) \; e^{\delta R_j^{t+1}}.$$
This now takes a more standard form as a multiplicative update; the $(1-\beta)$ term can be ignored as it is cancels out in $P_j^t$, so the main difference with the standard MWU is the $\frac{\sum_{k=1}^m W_k^t}{m}$ term that takes up a $\mu$ fraction of the weight.
One can think of this as a regularizing term and bears some similarity to what was considered in the recent breakthrough result on computing flows \cite{ChristianoKMST11}.

Assuming that $\eta_1\geq \eta_j$ for all $j \neq 1$, the optimal strategy for the population is to select option $1$.
If they did this, then the average expected gain of the population is 
$ \eta_1.$
On the other hand, the average expected gain of the population over $T$ iterations while following this stochastic process  is
$ \frac 1 T \sum_{t=1}^T \sum_{j=1}^m \E\left[  P_j^{t-1} R_j^{t} \right].$
We now proceed to understand how the latter compares to the former -- in other words, bound the regret of the infinite population stochastic process.
Formally, in the remainder of this subsection, we prove the following result:
\begin{theorem}[Regret of Infinite Population Distributed Learning Dynamics]\label{thm:main1}
Let $\eta_1>\eta_2 \geq \cdots \geq \eta_m$, let $\frac{1}{2} < \beta \leq \frac{e}{e+1}$ (and hence $0 < \delta \leq 1$), and let $6\mu \leq \delta^2$.
Let $P^0$ be the uniform distribution on $\{1,\ldots,m\}$ and $\{P^t\}_{t=1}^T$ be the probability distributions produced by the infinite population distributed learning dynamics with stochastic rewards $\{R^t\}_{t=1}^T$. 
Then for $T \geq \frac{\ln m}{\delta^2},$ the average expected regret after $T$ steps is 
\[ \mathrm{Regret}_{\infty}(T) = \eta_1 - \frac{1}{T} \cdot  \sum_{t=1}^T \sum_{j=1}^m \E\left[  P_j^{t-1} R_j^{t} \right] 
\leq 3\delta. \]
Furthermore, 
$ \frac{1}{T} \sum_{t=1}^T  \E \left[ P_1^{t-1} \right] \geq  1- \frac{3 \delta}{\eta_1-\eta_2}.$
\end{theorem}
The proof of \cref{thm:main1} appears in Section \ref{sec:proof}

\subsection{Main Result: Regret of the Distributed Learning Dynamics in Finite Populations}
\label{sec:main}

\begin{theorem}[Regret of the Distributed Learning Process in Finite Populations]\label{thm:main2}
Let $\eta_1>\eta_2 \geq \cdots \geq \eta_m$, let $\frac{1}{2} < \beta \leq \frac{e}{e+1}$ (and hence $0 < \delta \leq 1$), let $6\mu \leq \delta^2$, and let $\delta'':= \sqrt{ \frac{60 m \ln N}{(1-\beta)\mu N}}$, $c:=\frac{240 m}{(1-\beta) \mu},$ and $N$ is such that
\[ \frac{N}{\ln N} \geq \frac{\inp{c\frac{4m}{\mu(1-\beta)}}^{\frac{2\ln 5}{\delta^2}}}{\delta''^2}  
\mbox{ and } 
N^{10} \geq \frac{24m \ln m}{\mu (1-\beta) \delta^3}.\]
Let $Q^0$ be the uniform distribution on $\{1,\ldots,m\}$ and  $\{Q^t\}_{t=1}^T$ be the probability distributions produced by the finite population distributed learning dynamics with stochastic rewards $\{R^t\}_{t=1}^T$, 
then for $\frac{N^{10}}{m\delta} \geq T \geq \frac{\ln m}{\delta^2},$ the average expected regret after $T$ steps is 
$ \eta_1 - \frac{1}{T} \cdot  \sum_{t=1}^T \sum_{j=1}^m \E\left[  Q_j^{t-1} R_j^{t} \right] \leq 6\delta.$
\end{theorem}
Here, $N^{10}$ is arbitrary and can be made as large as required at the expense of constants.

In order to prove this result, we first give an analysis which holds for $T = \frac{\ln m}{\delta^2}$, and then show how to leverage this result for large $T$.

\subsubsection{\textbf{Small $T$}}
\label{sec:smallt}
To conduct the analysis, we first set up some definitions and which correspond to the two different stages of the finite population distributed learning process, and show that, in each step, the finite population dynamics  is approximated by the infinite population distributed learning dynamics/MWU stochastic process.
\paragraph{\bf Stage 1.}
Let $D_j^t$ denote the number of people committed to option $j$ at time $t$ and let $Q_j^t:= \frac{D_j^t}{\sum_{k=1}^m D_k^t}$ be the probability distribution capturing the relative popularity of option $j$ at time $t.$  
Let $\mathcal{S}_j^{t+1} \subseteq [N]$ denote the set of people who select $j$ after the first stage in the sampling process in step $t+1$ and let $S_j^{t+1}:=|\mathcal{S}_j^{t+1}|.$
Let $Y_{ij}^{t+1}$ be the indicator random variable for the event that $i$ chooses $j$ in stage one at time step $t+1$.
Note that these $\inb{Y_{ij}^{t+1}}_{i=1}^N$  are independent conditioned on everything up to time $t$ with
\begin{equation}\label{eq:low}
\Pr{ Y_{ij}^{t+1} =1 \; \middle| \; t } =  \inp{ (1-\mu) Q_j^t +\frac{\mu}{m} } \geq \frac{\mu}{m}.
\end{equation}
Since $S_j^{t+1} = \sum_{i=1}^N Y_{ij}^{t+1}$,
it follows from linearity of expectation that 
$ \Ex{ {S_j^{t+1}} \; \middle| \; t} = \inp{ (1-\mu) Q_j^t +\frac{\mu}{m} }N.$
\begin{proposition}\label{fct:1}
Let $t \geq 0$ be fixed. For $\delta':= \sqrt{ \frac{30 m \ln N}{\mu N}} \leq \frac{1}{2}$, with probability at least $1-\frac{2m}{N^{10}}$ (conditioned on everything up to time $t$),  for all $j$  
$  {S_j^{t+1}}  \stackrel{1+2\delta'}{\sim} { \inp{ (1-\mu) Q_j^t +\frac{\mu}{m}} N}.$
Thus, we deduce that (unconditionally) with  probability $1-\frac{2m}{N^{10}}$, $S_j^{t+1} \geq \frac{\mu N}{2m}$  for all $j.$
\end{proposition}

\begin{proof}
The proof follows from Chernoff-Hoeffding bound (\cref{thm:chernoff}),  noting that $\gamma \geq \frac{\mu}{m}$ from \eqref{eq:low}, and taking a union bound over all $j \in [m]$.
For the latter part, note that for any fixed $t$ and for all $j$, with probability at least $1-\frac{2m}{N^{10}},$  
$ S_j^{t+1} \geq \frac{1}{1+2 \delta'} \frac{\mu N}{m}.$ 
Since $\delta' \leq \frac{1}{2}$, this implies that with probability at least $1-\frac{2m}{N^{10}},$  we have  $ S_j^{t+1} \geq  \frac{\mu N}{2m}$.
\end{proof}

\paragraph{\bf Stage 2.}
There are two outcomes for each option $j$: $R_j^{t+1}=1$ and $R_j^{t+1}=0.$
Let $Z_{ij}^{t+1}$ be the indicator random variable for the event that $i$ commits to $j$ in stage two of the process.
Note that 
\begin{equation}\label{eq:low2}
 \Pr{ Z_{ij}^{t+1} =1 \; \middle| \; \mathcal{S}_j^{t+1}, R_j^{t+1},t} = \beta^{R_j^{t+1}} (1-\beta)^{1-R_j^{t+1}} \geq (1-\beta)
 \end{equation}
if $i \in \mathcal{S}_j^{t+1}$ since $\beta \geq \nicefrac 1 2$, and $\Pr{ Z_{ij}^{t+1} =1 \; \middle| \; \mathcal{S}_j^{t+1}, R_j^{t+1},t} = 0$ otherwise.
Let $D_j^{t+1} = \sum_{j \in \mathcal{S}_j^{t+1} } Z_{ij}^{t+1}.$
 It follows from linearity of expectation that 
$ \Ex{ D_j^{t+1} \; \middle| \; \mathcal{S}_j^{t+1}, R_j^{t+1},t} = {S_j^{t+1}} \beta^{R_j^{t+1}} (1-\beta)^{1-R_j^{t+1}}.$
\begin{proposition}\label{fct:2} Let $t \geq 0$ be fixed. For  $\delta'':= \sqrt{ \frac{60 m \ln N}{(1-\beta)\mu N}} \leq \frac{1}{2}$, with probability greater than $1-\frac{4m}{N^{10}}$ (conditioned on everything up to time $t, \mathcal{S}_j^{t+1},R_j^{t+1}$ ), for  and all $j$  
$  {D_j^{t+1}} \stackrel{1+2\delta''}{\sim} {{{S_j^{t+1}} \beta^{R_j^{t+1}} (1-\beta)^{1-R_j^{t+1}}}}.$
\end{proposition}
\begin{proof}
The proof again follows from Chernoff-Hoeffding bound  (\cref{thm:chernoff}),  noting that $\gamma \geq 1-\beta$ from \eqref{eq:low2}, and taking a union bound over all $j \in [m]$.
Here we have also used  \cref{fct:1} that $S_j^{t+1} \geq \frac{\mu N}{2m}$ for large enough $N$ for all $j$ with probability at least $1-\frac{2m}{N^{10}}.$
\end{proof}

Combining \cref{fct:1} and \cref{fct:2}  we obtain the following.
\begin{proposition}\label{fct:3} 
Let $t \geq 0$ be fixed. Let $\delta'':= \sqrt{ \frac{60 m \ln N}{(1-\beta)\mu N}}$.
With probability greater than $1-\frac{6m}{N^{10}}$ (conditioned on everything up to time $t,R_j^{t+1}$ )
 ${D_j^{t+1}}\stackrel{1+6\delta''}{\sim}{{ \inp{ (1-\mu) Q_j^t +\frac{\mu}{m}} N} \beta^{R_j^{t+1}} (1-\beta)^{1-R_j^{t+1}}}.$
\end{proposition}
\begin{proof}
Follows directly from \cref{fct:1} and \cref{fct:2} and noting that 
 $\delta' \leq \delta'' \leq \frac{1}{2}$. Thus 
$$ (1+2\delta') (1+2 \delta'') \leq 1+ 2 \delta' + 2 \delta'' + 4 \delta' \delta'' \leq 1+6\delta''.$$
\end{proof}

Now we establish a relationship between the probability distributions $P^t$ and $Q^t$. 
Note that both processes start with $P^0=Q^0$. 
Let $\delta_t:=5^t \delta''$.
\begin{lemma}\label{lem:main} There is a coupling such that 
$  {P_j^{t}} \stackrel{1+ \delta_t}{\sim}{Q_j^t} $
with probability at least $1-\frac{6tm}{N^{10}}$ for all choices of $\{R_j^t\}$s.
\end{lemma}
\begin{proof}
As suggested by our notation, we couple $P_j^t$ and $Q_j^t$ so that the realizations of  the $R_j^t$s is the same in both processes for all $j$ and all $t$.

The proof proceeds by induction on $t$. 
The statement holds when $t=0$ as, by definition, $P^0=Q^0$. 
Assume it holds for $t$.
Thus,  with probability at least $1-\frac{6tm}{N^{10}},$ ${P_j^{t}} \stackrel{1+\delta_t}{\sim}{Q_j^t}.$
  Let us condition on this event.
Recall that 
$$ P_j^{t+1} = \frac{ \left( (1-\mu) P_j^t + \frac{\mu}{m}  \right) \; \beta^{R_j^{t+1}} (1-\beta)^{1-R_j^{t+1}}}{\sum_{k=1}^m \left( (1-\mu) P_k^t + \frac{\mu}{m}  \right) \; \beta^{R_k^{t+1}} (1-\beta)^{1-R_k^{t+1}}}.
$$
Thus, with probability at least $1-\frac{6tm}{N^{10}},$
$$ P_j^{t+1} \stackrel{ (1+\delta_t)^2}{\sim} \frac{ \left( (1-\mu) Q_j^t + \frac{\mu}{m}  \right) \; \beta^{R_j^{t+1}} (1-\beta)^{1-R_j^{t+1}}}{\sum_{k=1}^m \left( (1-\mu) Q_k^t + \frac{\mu}{m}  \right) \; \beta^{R_k^{t+1}} (1-\beta)^{1-R_k^{t+1}}}.
$$
Here we used induction  for both the numerator and the denominator.
From  \cref{fct:2}, we know that 
$$ D_j^{t+1} \stackrel{1+6\delta ''}{\sim} \inp{ (1-\mu) Q_j^t +\frac{\mu}{m}} N \beta^{R_j^{t+1}} (1-\beta)^{1-R_j^{t+1}} $$
for all $j$ with probability at least $1-\frac{6m}{N^{10}}.$  
Thus, we obtain that with probability at least $1-\frac{6(t+1)m}{N^{10}}$
$$ 
P_j^{t+1} \stackrel{(1+\delta_t)^2(1+6\delta'')^2}{\sim} \frac{D_j^{t+1}}{\sum_{k=1}^m  D_k^{t+1}} = Q_j^{t+1}.
$$ 
Now note that, assuming that $\delta_t = 5^t \delta''\leq 1$ and $\delta'' \leq \frac{1}{40}$, 
$$ (1+\delta_t)^2 (1+6\delta'')^2 \leq (1+3 \delta_t) (1+13\delta'') \leq 1+ 3\delta_t + 13\delta''+39\delta''\delta_t  \leq 1+4 \delta_t + 13 \delta''\leq 1+5^{t+1}\delta'' =1+\delta_{t+1}$$
for $t \geq 2.$ For  $t=1$ the bound can be checked by a direct calculation. 
\end{proof}

Hence, for any fixed set of $R_j^{t}$s, 
the trajectories $P^t$ and $Q^t$ remain close. 
Thus, using \cref{thm:main1} and \cref{lem:main}, we obtain that 
$ \eta_1  - \inp{1+5^T \delta''} \frac{1}{T}  \sum_{t=1}^T \sum_{j=1}^m \E\left[  Q_j^{t-1} R_j^{t} \right] - \frac{6T^2m}{NT} \leq \frac{\ln m}{\delta T}+2 \delta.$
Here, the first negative term in the left hand side of the equation occurs when \cref{lem:main} applies and the second negative term is when it does not.
Rearranging, we obtain
$  \eta_1  -\frac{1}{T}  \sum_{t=1}^T \sum_{j=1}^m \E\left[  Q_j^{t-1} R_j^{t} \right] \leq \frac{\ln m}{\delta T}+2 \delta + 5^T \delta'' + \frac{6mT}{N^{10}}.$
Since $\delta'' := \sqrt{ \frac{240 m \ln N}{(1-\beta)\mu N}} \leq \sqrt{\frac{c \ln N}{N}}$ for $c=\frac{240 m}{(1-\beta) \mu}.$
Thus, when $T=\frac{\ln m}{\delta^2},$ 
$ 5^T \delta '' \leq \frac{m^\frac{\ln 5}{\delta^2}\sqrt{c \ln N}}{ \sqrt{{N}}}.$
Hence, when $N$ is such that 
\begin{equation}\label{eq:N}
\frac{N}{\ln N} \geq \frac{cm^{\frac{2\ln 5}{\delta^2}}}{\delta''^2}  \mbox{ 
and } N^{10} \geq \frac{6m \ln m}{\delta^3},
\end{equation}
then
\begin{equation}\label{eq:epoch-regret}
  \mbox{Regret}_{N}(T) =  \eta_1  -\frac{1}{T}  \sum_{t=1}^T \sum_{j=1}^m \E\left[  Q_j^{t-1} R_j^{t} \right] \leq 5 \delta.
 \end{equation}   
This concludes the proof of \cref{thm:main2} when  $T=\frac{\ln m}{\delta^2}$.

\subsubsection{\textbf{Large $T$}}
For large $T$ we need one new ingredient; here we focus on this additional aspect.
We first need a slight generalization of~\cref{thm:main1} to handle $P^0$ that  are not uniform. This is  similar to the version of MWU with {\em restricted distributions} as in Theorem 2.4 in \cite{arora2012multiplicative}.
\begin{theorem}[Regret of Infinite Population Distributed Learning Dynamics  with Nonuniform Start]\label{thm:main3}
Let $\eta_1 \geq\eta_2 \geq \cdots \geq \eta_m$, let $\frac{1}{2} < \beta \leq \frac{e}{e+1}$ (and hence $0 < \delta \leq 1$), and let $6\mu \leq \delta^2$. 
Let $P^0_j \geq \zeta$ for all $j \in \{1,\ldots,m\}$ and $\{P^t\}_{t=1}^T$ be the probability distributions produced by the MWU process with stochastic rewards $\{R^t\}_{t=1}^T$. 
Then for $T \geq \frac{\ln \inp{\nicefrac 1 \zeta}}{\delta^2},$ the average expected regret after $T$ steps is 
$ \mbox{Regret}_{\infty}(T) =  \eta_1 - \frac{1}{T} \cdot  \sum_{t=1}^T \sum_{j=1}^m \E\left[  P_j^{t-1} R_j^{t} \right] \leq 3\delta.$
\end{theorem}
The proof closely follows from that of \cref{thm:main1}, and we omit the details.

Similarly, the results in Section~\ref{sec:smallt} for small $T$ follow analogously with nonuniform start.
This just requires us to chose $N$ slightly bigger in particular, instead of \eqref{eq:N}, we would need
$N$  such that 
$ \frac{N}{\ln N} \geq \frac{c\inp{\frac{1}{\zeta}}^{\frac{2\ln 5}{\delta^2}}}{\delta''^2}  \mbox{ 
and } N^{10} \geq \frac{6 \ln m}{\zeta \delta^3}.$
This gives that the regret of the finite population distributed learning dynamics is at most $5 \delta$ when $T = \frac{\ln \inp{\nicefrac 1 \zeta}}{\delta^2}$.
We now complete the proof of \cref{thm:main2} for large $T$, by breaking the time into epochs consisting of $\frac{\ln \inp{\nicefrac 1 \zeta}}{\delta^2}$ time steps. 
In each epoch, we couple an infinite population distributed learning dynamics  with the finite population distributed learning dynamics such that the starting points are identical at the beginning of each epoch, and both observe the same sequence of rewards $R_j^t$s.

It remains to lower bound $\zeta$ appropriately. 
From \cref{fct:3}, it follows that for any $t$, with probability at least $1-\frac{6m}{N^{10}}$, for all $j$
 $ Q_j^t \geq \frac{\mu (1-\beta)}{4m}.
$
We let $\zeta:=\frac{\mu (1-\beta)}{4m}$, and hence
our epochs are of length 
$ \frac{\ln \inp{\frac{4m}{\mu(1-\beta)}}}{\delta^2}.$
This gives us regret 
$
    \eta_1  -\frac{1}{T}  \sum_{t=1}^T \sum_{j=1}^m \E\left[  Q_j^{t-1} R_j^{t} \right] \leq 5 \delta + \frac{Tm}{N^{10}}.
$
The additive term of $ \frac{Tm}{N^{10}}$ is precisely due to the fact that with probability $\frac{6m}{N^{10}}$,  the above inequality will not be satisfied, in which case the regret could be at most $1$ for that epoch.
Finally, note that $N^{10}$ is arbitrary and can be made as large as required at the expense of constants. 
This concludes the proof of \cref{thm:main2}.

\section{Proof of Theorem \ref{thm:main1}}\label{sec:proof}

Let us define the potential function 
$ \Phi^T := \sum_{j=1}^m W_j^T, $
and recall that $e^\delta = \frac{\beta}{1-\beta}.$
Then, 
\begin{eqnarray*}
\Phi^T &=&  \sum_{j=1}^m W_j^T \\
& = & (1-\beta) \sum_{j=1}^m  \left( (1-\mu) W_j^{T-1} + \frac{\mu}{m}  \sum_{k=1}^m W_k^{T-1}    \right) e ^{\delta R_j^T} \\
& = & (1-\beta) \Phi^{T-1}
\sum_{j=1}^m  \left( (1-\mu) P_j^{T-1} + \frac{\mu}{m}    \right) e ^{\delta R_j^T} \\
& \stackrel{0 \leq R_j^T \leq 1 \; , \; \cref{fct:exp}}{\leq} & (1-\beta) 
\Phi^{T-1} \sum_{j=1}^m  \left( (1-\mu) P_j^{T-1} + \frac{\mu}{m}    \right) \left(1+ (e^\delta -1) R_j^T \right) \\
& = & (1-\beta) \Phi^{T-1}
\left( 1+   (e^\delta -1)\sum_{j=1}^m  \left( (1-\mu) P_j^{T-1} + \frac{\mu}{m} \right)  R_j^T \right)  \\
& \stackrel{0 \leq R_j^T \leq 1}{\leq} & (1-\beta) \Phi^{T-1}
\left( 1+ \mu(e^\delta -1)+ (1-\mu)     (e^\delta -1) \sum_{j=1}^m   P_j^{T-1} R_j^T \right).
\end{eqnarray*}

Now, we let $\delta':=  \frac{(1-\mu)(e^\delta -1)}{1+\mu \delta}$ and obtain that 
\begin{eqnarray*}
\Phi^T & \stackrel{\delta \leq e^\delta-1}{\leq} & (1-\beta) (1+\mu(e^\delta -1)) \Phi^{T-1}
\left( 1 +  \frac{(1-\mu)(e^\delta -1)}{1+\mu \delta}\sum_{j=1}^m   P_j^{T-1} R_j^T \right)  \\
& \stackrel{  1 + \delta' x \leq e^{\delta' x}}{\leq} & (1-\beta) (1+\mu(e^\delta -1)) \Phi^{T-1} e^{\delta' \sum_{j=1}^m P_j^{T-1}R_j^T} \\
& \stackrel{\Phi^0=m}{\leq} &  (1-\beta)^T (1+\mu(e^\delta -1))^T  m  e^{{\delta'}\sum_{t=1}^T \sum_{j=1}^m P_j^{t-1}R_j^{t}}.
\end{eqnarray*}
On the other hand,
$$
\Phi^T \geq  (1-\beta)^T (1-\mu)^T e^{\delta \sum_{t=1}^T R_1^t}.
$$
  
Combining the lower bound and upper bound and taking logarithms we obtain
$$ \delta \sum_{t=1}^T R_1^t  \leq \ln m + T\ln \left(\frac{ 1+\mu(e^\delta -1)}{1-\mu}\right) + {\delta'}\sum_{t=1}^T \sum_{j=1}^m P_j^{t-1}R_j^{t} .$$

Thus,
\[ {\delta \sum_{t=1}^T R_1^t -  \delta'\sum_{t=1}^T \sum_{j=1}^m P_j^{t-1}R_j^{t}}  \leq \ln m + T\ln \left(\frac{ 1+\mu(e^\delta -1)}{1-\mu}\right) .\]

Now, for $\mu \leq \frac{1}{2}$ we know that $\frac{1}{1-\mu} \leq 1+2\mu$.
Also, $1+\mu(e^\delta -1) \leq 1+\mu(e-1).$ 
Thus, as $\mu \leq \frac{1}{2},$ 
\[\frac{ 1+\mu(e^\delta -1)}{1-\mu} \leq (1+(e-1)\mu)(1+2\mu) \leq 1+ (e+1) \mu +2(e-1)\mu^2 \leq 1+ 2e \mu \leq 1+6\mu.\]
Hence, using the inequality $\ln (1+x) \leq x$ for all $x \geq 0$, we obtain
\[\ln \left(\frac{ 1+\mu(e^\delta -1)}{1-\mu}\right) \leq 6\mu.\]

Further, using the fact  that $ e^\delta -1 \leq \delta+\delta^2$ for $0 \leq \delta \leq 1,$ which is implied by the assumption that $\beta \leq \frac{e}{1+e},$  it can be seen that
\[ \delta' =  \frac{(1-\mu)(e^\delta -1)}{1+\mu \delta} \leq  \frac{(1-\mu)\delta (1+ \delta )}{1+\mu \delta}  \leq \delta (1+\delta).\]
Therefore,
\[ \delta \inp{ \sum_{t=1}^T R_1^t -  (1+\delta)\sum_{t=1}^T \sum_{j=1}^m P_j^{t-1}R_j^{t}}   \leq \ln m + 6\mu T.\]

Note that $\sum_{t=1}^T  \sum_{j=1}^m P_j^{t-1}R_j^{t} \leq T$, hence, the above implies that
\[ \delta \inp{ \sum_{t=1}^T R_1^t - \sum_{t=1}^T \sum_{j=1}^mP_j^{t-1}R_j^{t}}   \leq \ln m + (\delta^2+6\mu) T .\]

Taking expectations and dividing by $T \delta$  we obtain the following regret bound.
$$ \eta_1 -   \frac{1}{T} \cdot  \sum_{t=1}^T \sum_{j=1}^m \E\left[  P_j^{t-1} R_j^{t} \right] \leq \frac{\ln m}{\delta T} + \inp{\delta + \frac{6\mu}{\delta}}.$$
Assuming $6\mu \leq \delta^2$ we obtain 
$$ \eta_1 - \frac{1}{T} \cdot  \sum_{t=1}^T \sum_{j=1}^m \E\left[  P_j^{t-1} R_j^{t} \right] \leq \frac{\ln m}{\delta T} + 2\delta.$$
Thus, for $T \geq \frac{\ln m}{\delta^2},$ we obtain the desired bound
$$ \mbox{Regret}_{\infty}(T)   \leq 3\delta.$$

From this we can derive the lower bound on the probability that the best option $j=1$ is selected as stated in the second part of the theorem.
Firstly, it follows (as $R_j^t$ is independent of $P_j^{t-1}$) that for all $T \geq 1$
$$  \eta_1  -  \frac{1}{T} \cdot  \sum_{t=1}^T \sum_{j=1}^m \eta_j \E\left[  P_j^{t-1}  \right] \leq \frac{\ln m}{\delta T} + 2\delta.$$
Thus,
$$ \eta_1 \left( 1-\frac{1}{T} \cdot  \sum_{t=1}^T  \E\left[  P_1^{t-1}  \right] \right) - \frac{\eta_2}{T} \cdot  \sum_{t=1}^T \sum_{j=2}^m  \E\left[  P_j^{t-1}  \right] \leq \frac{\ln m}{\delta T}+2\delta.$$
From this we obtain 
\[ (\eta_1 - \eta_2) \left( 1- \frac{1}{T} \cdot  \sum_{t=1}^T  \E\left[  P_1^{t-1}  \right] \right) \leq \frac{\ln m}{\delta T}+ 2\delta,\]
and consequently for $T \geq \frac{\ln m}{\delta^2},$ 
\[  \frac{1}{T} \sum_{t=1}^T  \E \left[ P_1^{t-1} \right] \geq  1- \frac{3 \delta}{\eta_1-\eta_2}.\]
This completes the proof.

\section{Conclusion and Future Work}
In this work we study  a fundamental distributed learning dynamics prevalent in various social and biological contexts and provide the first convergence and regret bounds for it in the finite population setting.
The connection between this learning dynamics and the MWU method suggests a novel distributed and essentially memoryless implementation of the MWU method. 
Another interpretation of our result comes by looking at the infinite population limit of the distributed learning dynamics; while an individual can be effectively solving a stochastic multi-armed bandit problem, the population as a whole is solving a full-information version of the problem, and hence can be very efficient on the group-level. 

Several important directions remain open.
The first is to extend our results to the social network setting where individuals can only sample in step (1) from their neighbors.
The question here would be whether, and to what extent, the efficiency of the group remains as a function of the network topology. 
It would also be interesting to explore the distributed learning algorithms when  the parameters controlling the quality of the options ($\eta_i$s) are allowed to change, or when there is dependence across options and time (e.g., when the options represent stocks). 
Lastly, we note that as an algorithm designer, if we were to implement these learning dynamics as a distributed approximation to the stochastic version of MWU method, we can optimize $\beta$ to attain the usual $O\inp{\sqrt{{\ln m}/{T}}}$ regret; in the distributed learning dynamics, we are constrained by the behavior of the group -- the regret bound will only be as good as the $\beta$ they use. 
 This naturally raises the question of whether human groups match the ideal values for $\beta$, perhaps in a context-specific manner, to achieve good regret bounds.

\section*{Acknowledgments}

The authors would like to thank Ashish Goel for useful discussions, and the BIRS-CMO 2016 Workshop on Models and Algorithms for Crowds and Networks, where part of this work was done.

\bibliographystyle{plain}
\bibliography{references}

\end{document}